\documentclass[11pt]{article}

\usepackage{fullpage,times}

\usepackage{amsthm,amsfonts,amsmath,amssymb,epsfig,color,float,graphicx,verbatim}
\usepackage{algorithm,algorithmic}

\newtheorem{theorem}{Theorem}

\newtheorem{lemma}{Lemma}

\newtheorem{remark}{Remark}

\newcommand{\reals}{\mathbb{R}}
\newcommand{\E}{\mathbb{E}}

\newcommand{\bx}{\mathbf{x}}
\newcommand{\bw}{\mathbf{w}}
\newcommand{\bg}{\mathbf{g}}

\newcommand{\bu}{\mathbf{u}}
\newcommand{\bv}{\mathbf{v}}

\newcommand{\Ocal}{\mathcal{O}}

\newcommand{\norm}[1]{\|#1\|}
\newcommand{\inner}[1]{\langle#1\rangle}

\newcommand{\secref}[1]{Sec.~\ref{#1}}

\newcommand{\figref}[1]{Fig.~\ref{#1}}
\renewcommand{\eqref}[1]{Eq.~(\ref{#1})}
\newcommand{\lemref}[1]{Lemma~\ref{#1}}

\newcommand{\thmref}[1]{Thm.~\ref{#1}}

\title{A Stochastic PCA and SVD Algorithm\\ with an Exponential Convergence Rate}
\author{Ohad Shamir\\Weizmann Institute of Science\\\texttt{ohad.shamir@weizmann.ac.il}}
\date{}

\begin{document}

\maketitle

\begin{abstract}
We describe and analyze a simple algorithm for principal component analysis
and singular value decomposition, VR-PCA, which uses computationally cheap
stochastic iterations, yet converges exponentially fast to the optimal
solution. In contrast, existing algorithms suffer either from slow
convergence, or computationally intensive iterations whose runtime scales
with the data size. The algorithm builds on a recent variance-reduced
stochastic gradient technique, which was previously analyzed for strongly
convex optimization, whereas here we apply it to an inherently non-convex
problem, using a very different analysis.
\end{abstract}

\section{Introduction}\label{sec:introduction}

We consider the following fundamental matrix optimization problem: Given a
matrix $X\in \reals^{d\times n}$, we wish to recover its top $k$ left
singular vectors (where $k\ll d$) by solving
\begin{equation}\label{eq:optgeneralproblem}
\max_{W\in \reals^{d\times k}:W^\top W = I}~\frac{1}{n}\norm{X^\top W}_F^2,
\end{equation}
$\norm{\cdot}_F$ being the Frobenius norm and $I$ being the identity
matrix\footnote{The top $k$ right singular values can also be extracted, by
considering the matrix $X^\top$ in lieu of $X$.}. A prominent application in
machine learning and statistics is Principal Component Analysis (PCA), which
is one of the most common tools for unsupervised data analysis and
preprocessing: Given a data matrix $X$ whose columns consist of $n$ instances
in $\reals^d$, we are interested in finding a $k$-dimensional subspace
(specified by a $d\times k$ matrix $W$), on which the projection of the data
has largest possible variance. Finding this subspace has numerous uses, from
dimensionality reduction and data compression to data visualization, and the
problem is extremely well-studied.

Letting $\bx_1,\ldots,\bx_n$ denote the columns of $X$,
\eqref{eq:optgeneralproblem} can be equivalently written as
\begin{equation}\label{eq:optproblem}
\min_{W\in\reals^{d\times k}:W^\top
W=I}-W^\top\left(\frac{1}{n}\sum_{i=1}^{n}\bx_i\bx_i^\top\right)W,
\end{equation}
which reveals that the solution is also the top $k$ eigenvectors of the
covariance matrix $\frac{1}{n}\sum_{i=1}^{n}\bx_i\bx_i^\top$. In this paper,
we will mostly focus on the simplest possible form of this problem, where
$k=1$, in which case the above reduces to
\begin{equation}\label{eq:optproblem1}
\min_{\bw:\norm{\bw}_2=1}
-\bw^\top\left(\frac{1}{n}\sum_{i=1}^{n}\bx_i\bx_i^\top\right)\bw,
\end{equation}
and our goal is to find the top eigenvector $\bv_1$. However, as discussed
later, the algorithm to be presented can be readily extended to solve
\eqref{eq:optproblem} for $k>1$.

When the data size $n$ and the dimension $d$ are modest, this problem can be
solved exactly by a full singular value decomposition of $X$. However, the
required runtime is $\Ocal\left(\min\{nd^2,n^2d\}\right)$, which is
prohibitive in large-scale applications. A common alternative is to use
iterative methods such as power iterations or more sophisticated variants
\cite{GolvaL12}.
If the covariance matrix has bounded spectral norm and an eigengap $\lambda$
between its first and second eigenvalues, then these algorithms can be shown
to produce a unit vector which is $\epsilon$-far from $\bv_1$ (or $-\bv_1$)
after $\Ocal\left(\frac{\log(1/\epsilon)}{\lambda^p}\right)$ iterations
(where e.g. $p=1$ for power iterations). However, each iteration involves
multiplying one or more vectors by the covariance matrix $\frac{1}{n}\sum_i
\bx_i\bx_i^\top$. Letting $d_s\in [0,d]$ denote the average sparsity (number
of non-zero entries) in each $\bx_i$, this requires $\Ocal(d_s n)$ time by
passing through the entire data. Thus, the total runtime is
$\Ocal\left(\frac{d_s n\log(1/\epsilon)}{\lambda^p}\right)$. When $\lambda$
is small, this is equivalent to many passes over the data, which can be
prohibitive for large datasets.

An alternative to these deterministic algorithms are stochastic and
incremental algorithms (e.g.
\cite{krasulina1969method,oja1982simplified,oja1985stochastic} and more
recently,
\cite{ACLS12,mitliagkas2013memory,arora2013stochastic,hardt2014noisy,de2014global}).
In contrast to the algorithms above, these algorithms perform much cheaper
iterations by choosing some $\bx_i$ (uniformly at random or otherwise), and
updating the current iterate using only $\bx_{i}$. In general, the runtime of
each iteration is only $\Ocal(d_s)$. On the flip side, due to their
stochastic and incremental nature, the convergence rate (when known) is quite
slow, with the number of required iterations scaling linearly with
$1/\epsilon$ and additional problem parameters. This is useful for getting a
low to medium-accuracy solution, but is prohibitive when a high-accuracy
solution is required.

In this paper, we propose a new stochastic algorithm for solving
\eqref{eq:optproblem1}, denoted as VR-PCA \footnote{VR stands for
``variance-reduced''.}, which for bounded data and under suitable
assumptions, has provable runtime of
\[
\Ocal\left(d_s\left(n+\frac{1}{\lambda^2}\right)\log\left(\frac{1}{\epsilon}\right)\right).
\]
This algorithm combines the advantages of the previously discussed
approaches, while avoiding their main pitfalls: On one hand, the runtime
depends only logarithmically on the accuracy $\epsilon$, so it is suitable to
get high-accuracy solutions; while on the other hand, the runtime scales as
the \emph{sum} of the data size $n$ and a factor involving the eigengap
parameter $\lambda$, rather than their product. This means that the algorithm
is still applicable when $\lambda$ is relatively small. In fact, as long as
$\lambda\geq \Omega(1/\sqrt{n})$, this runtime bound is better than those
mentioned earlier, and equals $d_s n$ up to logarithmic factors: Proportional
to the time required to perform a single scan of the data.

VR-PCA builds on a recently-introduced technique for stochastic gradient
variance reduction (see \cite{johnson2013accelerating} as well as
\cite{MahZha13,KonRi13}, and \cite{FrosGeKaSi14} in a somewhat different
context). However, the setting in which we apply this technique is quite
different from previous works, which crucially relied on the strong convexity
of the optimization problem, and often assume an unconstrained domain. In
contrast, our algorithm attempts to minimize the function in
\eqref{eq:optproblem1}, which is nowhere convex, let alone strongly convex
(in fact, it is \emph{concave} everywhere). As a result, the analysis in
previous papers is inapplicable, and we require a new and different analysis
to understand the performance of the algorithm.

\section{Algorithm and Analysis}\label{sec:alg}

The pseudo-code of our algorithm appears as Algorithm \ref{alg:alg} below. We
refer to a single execution of the inner loop as an \emph{iteration}, and
each execution of the outer loop as an \emph{epoch}. Thus, the algorithm
consists of several epochs, each of which consists of running $m$ iterations.

\begin{center}
\begin{minipage}{0.6\textwidth}
\begin{algorithm}[H]
\caption{VR-PCA} \label{alg:alg}
\begin{algorithmic}
\STATE \textbf{Parameters:} Step size $\eta$, epoch length $m$ \STATE
\textbf{Input:} Data matrix $X=(\bx_1,\ldots,\bx_n)$; Initial unit vector
$\tilde{\bw}_0$ \FOR{$s=1,2,\ldots$}
  \STATE $\tilde{\bu}=\frac{1}{n}\sum_{i=1}^{n}\bx_i\left(\bx_i^\top \tilde{\bw}_{s-1}\right)$
  \STATE $\bw_0=\tilde{\bw}_{s-1}$
  \FOR{$t=1,2,\ldots,m$}
    \STATE Pick $i_t\in \{1,\ldots,n\}$ uniformly at random
    \STATE $\bw'_{t}=\bw_{t-1}+\eta\left(\bx_{i_t}\left(\bx_{i_t}^\top\bw_{t-1}-\bx_{i_t}^\top\tilde{\bw}_{s-1}\right)+\tilde{\bu}\right)$
    \STATE $\bw_{t}=\frac{1}{\norm{\bw'_t}}\bw'_{t}$
  \ENDFOR
  \STATE $\tilde{\bw}_{s}=\bw_m$
\ENDFOR
\end{algorithmic}
\end{algorithm}
\end{minipage}
\end{center}

To understand the structure of the algorithm, it is helpful to consider first
the well-known Oja's algorithm for stochastic PCA optimization
\cite{oja1982simplified}, on which our algorithm is based. In our setting,
this rule is reduced to repeatedly sampling $\bx_{i_t}$ uniformly at random,
and performing the update
\[
\bw'_{t} = \bw_{t-1}+\eta_t\bx_{i_t}\bx_{i_t}^\top\bw_{t-1}~~,~~
\bw_{t} = \frac{1}{\norm{\bw'_{t}}}\bw_{t}.
\]
Letting $A=\frac{1}{n}XX^\top=\frac{1}{n}\sum_{i=1}^{n}\bx_i\bx_i^\top$, this
can be equivalently rewritten as
\begin{equation}\label{eq:ojaalg}
\bw'_{t} = (I+\eta_t A)\bw_{t-1}+\eta_t\left(\bx_{i_t}\bx_{i_t}^\top-A\right)\bw_{t-1}~~,~~
\bw_{t} = \frac{1}{\norm{\bw'_{t}}}\bw_{t}.
\end{equation}
Thus, at each iteration, the algorithm performs a power iteration (using a
shifted and scaled version of the matrix $A$), adds a stochastic zero-mean
term $\eta_t\left(\bx_{i_t}\bx_{i_t}^\top-A\right)\bw_{t-1}$, and projects
back to the unit sphere. Recently, \cite{balsubramani2013fast} gave a
rigorous finite-time analysis of this algorithm, showing that if
$\eta_t=\Ocal(1/t)$, then under suitable conditions, we get a convergence
rate of $\Ocal(1/T)$ after $T$ iterations.

The reason for the relatively slow convergence rate of this algorithm is the
constant variance of the stochastic term added in each step. Inspired by
recent variance-reduced stochastic methods for convex optimization
\cite{johnson2013accelerating}, we change the algorithm in a way which
encourages the variance of the stochastic term to decay over time.
Specifically, we can rewrite the update in each iteration of our VR-PCA
algorithm as
\begin{equation}\label{eq:ouralg}
\bw'_{t} = (I+\eta A)\bw_{t-1}+\eta\left(\bx_{i_t}\bx_{i_t}^\top-A\right)\left(\bw_{t-1}-\tilde{\bw}_{s-1}\right)~~,~~
\bw_{t} = \frac{1}{\norm{\bw'_{t}}}\bw_{t},
\end{equation}
where $\tilde{\bw}_{s-1}$ is the vector computed at the beginning of each
epoch. Comparing \eqref{eq:ouralg} to \eqref{eq:ojaalg}, we see that our
algorithm also performs a type of power iteration, followed by adding a
stochastic zero-mean term. However, our algorithm picks a fixed step size
$\eta$, which is more aggressive that a decaying step size $\eta_t$.
Moreover, the variance of the stochastic term is no longer constant, but
rather controlled by $\norm{\bw_{t-1}-\tilde{\bw}_{s-1}}$. As we get closer
to the optimal solution, we expect that both $\tilde{\bw}_{s-1}$ and
$\bw_{t-1}$ will be closer and closer to each other, leading to decaying
variance, and a much faster convergence rate, compared to Oja's algorithm.

Before continuing to the algorithm's analysis, we make two important remarks:

\begin{remark}\label{remark:k1}
To generalize the algorithm to find multiple singular vectors (i.e. solve
\eqref{eq:optproblem} for $k>1$), one option is to replace the vectors
$\bw_t,\bw'_t,\tilde{\bw},\tilde{\bu}$ by $d\times k$ matrices
$W_t,W'_t,\tilde{W},\tilde{U}$, and replace the normalization step
$\frac{1}{\norm{\bw'_t}}\bw'_t$ by an orthogonalization step\footnote{I.e.
given $W'_t$, return $W_t$ with the same column space such that $W_t^\top
W_t=I$. Note that the algorithm relies on $W_t$ remaining parameterically
close to previous iterates, and $W'_t$ is a relatively small perturbation of
of an orthogonal $W_{t-1}$. Therefore, it's important to use an
orthogonalization procedure such that $W_t$ is close to $W'_t$ if $W'_t$ is
nearly orthogonal, such as Gram-Schmidt.}. This generalization is completely
analogous to how iterative algorithms such as power iterations and Oja's
algorithm are generalized to the $k>1$ case, and the same intuition discussed
above still holds. This is also the option used in our experiments. Another
option is to recover the singular vectors one-by-one via matrix deflation:
First recover the leading vector $\bv_1$, compute its associated eigenvalue $s_1$, and then iteratively recover the
leading eigenvector and eigenvalue of the deflated matrix
$\frac{1}{n}\sum_{i=1}^{n}\bx_i\bx_i^\top-\sum_{l=1}^{j-1}s_l\bv_l\bv_l^\top$,
which is precisely $\bv_j$. This is a standard method to extend power
iteration algorithms to recover multiple eigenvectors, and our algorithm can
be applied to solve it. Algorithmically, one simply needs to replace each
computation of the form $\bx\bx^\top\bw$ with
$\left(\bx\bx^\top-\sum_{l=1}^{j-1}\bv_l\bv_l^\top\right)\bw$. A disadvantage
of this approach is that it requires a positive eigengap between all top $k$
singular values, otherwise our algorithm is not guaranteed to converge.
\end{remark}

\begin{remark}\label{remark:sparse}
Using a straightforward implementation, the runtime of each iteration is
$\Ocal(d)$, and the total runtime of each epoch is $\Ocal(dm+d_s n)$, where
$d_s$ is the average sparsity of the data points $\bx_i$. However, a more
careful implementation can improve this to $\Ocal(d_s(m+n))$. The trick is to
maintain each $\bw_t$ as $\alpha\bg+\beta\tilde{\bu}$, plus a few additional
scalars, and in each iteration perform only a sparse update of $\bg$, and
updates of the scalars, all in $\Ocal(d_s)$ amortized time. See Appendix
\ref{app:sparse} for more details.
\end{remark}

A formal analysis of the algorithm appears as \thmref{thm:main} below. See
\secref{sec:experiments} for further discussion of the choice of parameters
in practice.

\begin{theorem}\label{thm:main}
  Define $A$ as $\frac{1}{n}XX^\top=\frac{1}{n}\sum_{i=1}^{n}\bx_i\bx_i^\top$, and let $\bv_1$ be an eigenvector corresponding to its largest eigenvalue. Suppose that
  \begin{itemize}
    \item $\max_i\norm{\bx_i}^2\leq r$ for some $r>0$.
    \item $A$ has eigenvalues $s_1>s_2\geq\ldots\geq s_d$, where $s_1-s_2=\lambda$ for some $\lambda>0$.
    \item $\inner{\tilde{\bw}_0,\bv_1}\geq \frac{1}{\sqrt{2}}$.
  \end{itemize}

  Let $\delta,\epsilon\in (0,1)$ be fixed. If we run the algorithm with any epoch length parameter $m$ and step size $\eta$, such that
  \begin{equation}\label{eq:thmcondme}
\eta \leq \frac{c_1\delta^2}{r^2}\lambda~~~~,~~~~
m\geq \frac{c_2\log(2/\delta)}{\eta \lambda}
~~~~,~~~~m\eta^2r^2+r\sqrt{m\eta^2\log(2/\delta)}\leq c_3,
  \end{equation}
  (where $c_1,c_2,c_3$ designates certain positive numerical constants), and for $T=\left\lceil\frac{\log(1/\epsilon)}{\log(2/\delta)}\right\rceil$ epochs,
  then with probability at least $1-2\log(1/\epsilon)\delta$, it holds that
  \[
  \inner{\tilde{\bw}_T,\bv_1}^2 ~\geq~ 1-\epsilon.
  \]
\end{theorem}

The proof of the theorem is provided in \secref{sec:proof}. It is easy to
verify that for any fixed $\delta$, \eqref{eq:thmcondme} holds for any
sufficiently large $m$ on the order of $\frac{1}{\eta\lambda}$, as long as
$\eta$ is chosen to be sufficiently smaller than $\lambda/r^2$. Therefore, by
running the algorithm for $m=\Theta\left(\left(r/\lambda\right)^2\right)$
iterations per epoch, and $T=\Theta(\log(1/\epsilon))$ epochs, we get
accuracy $\epsilon$ with high probability\footnote{Strictly speaking, this
statement is non-trivial only in the regime of $\epsilon$ where
$\log\left(\frac{1}{\epsilon}\right)\ll \frac{1}{\delta}$, but if $\delta$ is
a reasonably small $(\ll 1)$, then this is the practically relevant regime.
Moreover, as long as the success probability is positive, we can get an
algorithm which succeeds with exponentially high probability by an
amplification argument: Simply run several independent instantiations of the
algorithm, and pick the solution $\bw$ for which
$\bw^\top\left(\frac{1}{n}\sum_{i=1}^{n}\bx_i\bx_i^\top\right)\bw$ is
largest.} $1-2\log(1/\epsilon)\delta$. Since each epoch requires
$\Ocal(d_s(m+n))$ time to implement, we get a total runtime of
\begin{equation}\label{eq:runtime}
\Ocal\left(d_s\left(n+\left(\frac{r}{\lambda}\right)^2\right)\log\left(\frac{1}{\epsilon}\right)\right),
\end{equation}
establishing an exponential convergence rate. If $\lambda/r \geq
\Omega(1/\sqrt{n})$, then the runtime is $\Ocal(d_s n\log(1/\epsilon))$ -- up
to log-factors, proportional to the time required just to scan the data once.

The theorem assumes that we initialize the algorithm with $\tilde{\bw}_0$ for
which $\inner{\tilde{\bw}_0,\bv_1}\geq \frac{1}{\sqrt{2}}$. This is not
trivial, since if we have no prior knowledge on $\bv_1$, and we choose
$\tilde{\bw}_0$ uniformly at random from the unit sphere, then it is
well-known that $|\inner{\tilde{\bw}_0,\bv_1}|\leq \Ocal(1/\sqrt{d})$ with
high probability. Thus, the theorem should be interpreted as analyzing the
algorithm's convergence after an initial ``burn-in'' period, which results in
some $\tilde{\bw}_0$ with a certain constant distance from $\bv_1$. This
period requires a separate analysis, which we leave to future work. However,
since we only need to get to a constant distance from $\bv_1$, the runtime of
that period is independent of the desired accuracy $\epsilon$. Moreover, we
note that in our experiments (see \secref{sec:experiments}), even when
initialized from a random point, no ``burn-in'' period is discernable, and
the algorithm seems to enjoy the same exponential convergence rate starting
from the very first epoch. Finally, since the variance-reduction technique
only kicks in once we are relatively close to the optimum, it is possible to
use some different stochastic algorithm with finite-time analysis, such as
Oja's algorithm (e.g. \cite{balsubramani2013fast}) or
\cite{hardt2014noisy,de2014global} to get to this constant accuracy, from
which point our algorithm and analysis takes over (for example, the algorithm
of \cite{de2014global} would require $\Ocal(d/\lambda^2)$ iterations,
starting from a randomly chosen point, according to their analysis). In any
case, note that some assumption on $\inner{\tilde{\bw}_0,\bv_1}$ being
bounded away from $0$ must hold, otherwise the algorithm may fail to converge
in the worst-case (a similar property holds for power iterations, and follows
from the non-convex nature of the optimization problem).

\section{Experiments}\label{sec:experiments}

We now turn to present some experiments, which demonstrate the performance of
the VR-PCA algorithm. Rather than tuning its parameters, we used the
following fixed heuristic: The epoch length $m$ was set to $n$ (number of
data points, or columns in the data matrix), and $\eta$ was set to
$\eta=\frac{1}{\bar{r}\sqrt{n}}$, where
$\bar{r}=\frac{1}{n}\sum_{i=1}^{n}\norm{\bx_i}^2$ is the average squared norm
of the data. The choice of $m=n$ ensures that at each epoch, the runtime is
about equally divided between the stochastic updates and the computation of
$\tilde{\bu}$. The choice of $\eta$ is motivated by our theoretical analysis,
which requires $\eta$ on the order of $1/(\max_i\norm{\bx_i}^2\sqrt{n})$ in
the regime where $m$ should be on the order of $n$. Also, note that this
choice of $\eta$ can be readily computed from the data, and doesn't require
knowledge of $\lambda$.

\begin{figure}[ht]
\begin{center}
  \includegraphics[trim = 3cm 1cm 3cm 0.4cm, clip=true, scale=0.5]{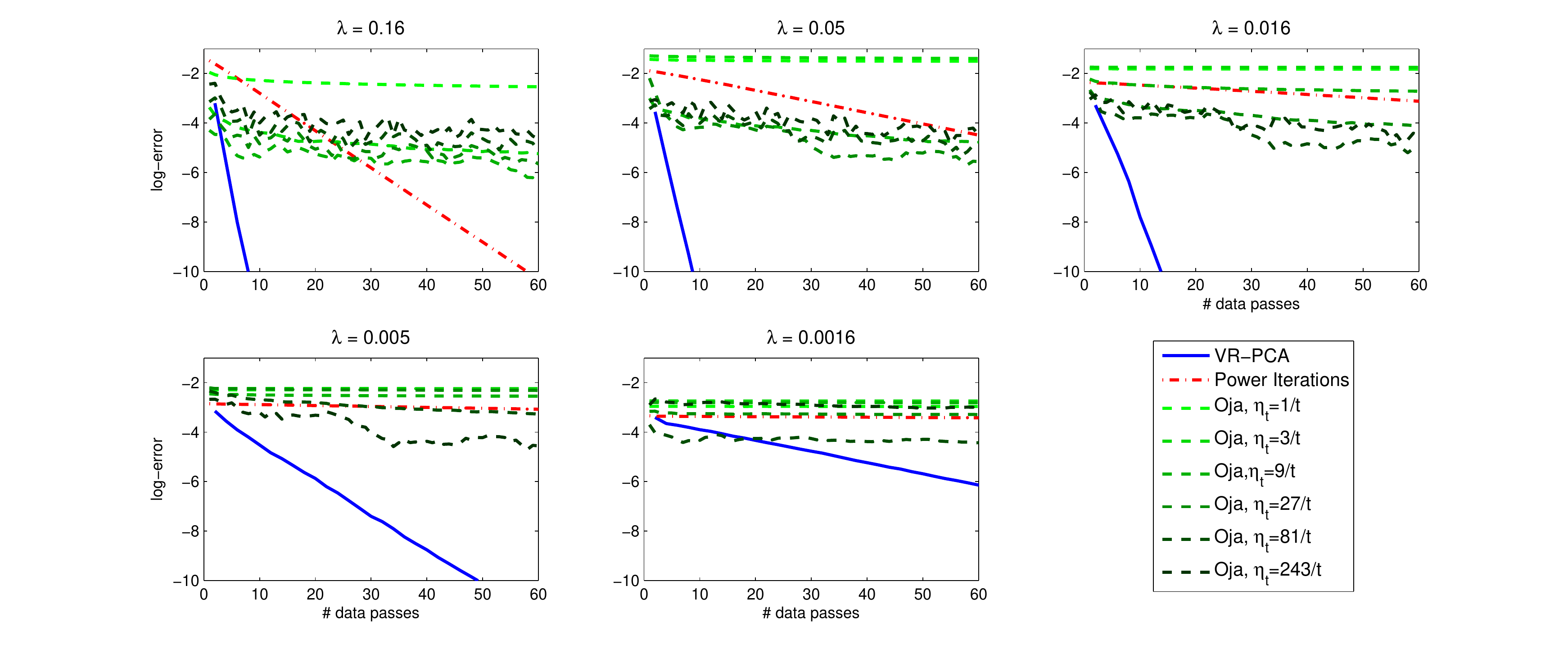}
\end{center}
  \caption{Results for synthetic data. Each plot represents results for a single dataset with eigengap $\lambda$,
  and compares the performance of VR-PCA to power iterations and Oja's algorithm with different step sizes $\eta_t$. In each
  plot, the x-axis represents the number of effective data passes (assuming $2$ per epoch
  for VR-PCA), and the y-axis equals $\log_{10}\left(1-\frac{\norm{X^\top \bw}^2}{\max_{\bv:\norm{\bv}=1}\norm{X^\top \bv}^2}\right)$, where $\bw$
  is the vector obtained so far.}
  \label{fig:synthetic}
\end{figure}

\begin{figure}[ht]
\begin{center}
  \includegraphics[scale=0.6]{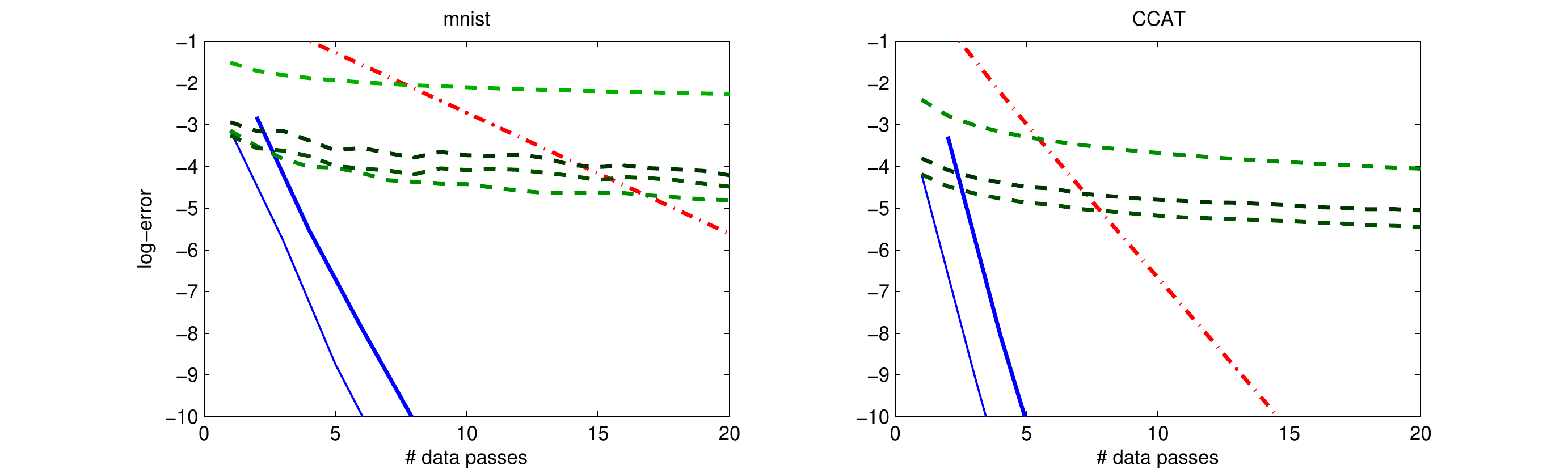}
\end{center}
  \caption{Results for the MNIST and CCAT datasets, using the same algorithms as in \figref{fig:synthetic},
  as well as the hybrid method described in the text (represented by a thinner plain line).
  See \figref{fig:synthetic} for a legend.}
  \label{fig:real}
\end{figure}

\begin{figure}[ht]
\begin{center}
  \includegraphics[scale=0.56]{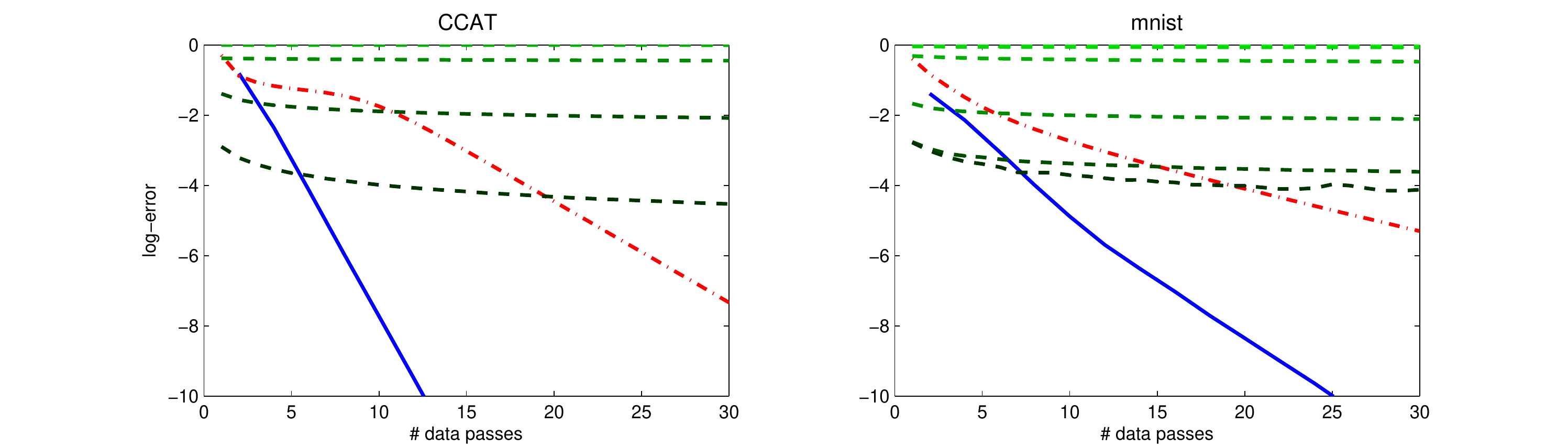}
\end{center}
  \caption{Results for MNIST (for $k=6$ singular vectors) and CCAT (for $k=3$ singular vectors).
  The y-axis here equals $\log_{10}\left(1-\frac{\norm{X^\top W}_F^2}{\max_{V:V^\top V=I}\norm{X^\top V}_F^2}\right)$,
  with $W\in \reals^{d\times k}$ being the current
  iterate. This directly generalizes the performance measure used in previous figures for the $k>1$ case.
  See \figref{fig:synthetic} for a legend.}
  \label{fig:realblock}
\end{figure}

First, we performed experiments on several synthetic random datasets (where
$n=200000,d=10000$), with different choices of eigengap\footnote{For each
choice of $\lambda$, we constructed a $d\times d$ diagonal matrix $D$, with
diagonal $(1,1-\lambda,1-1.1\lambda,\ldots,1-1.4\lambda,q_1,q_2,\ldots)$
where $q_i=|g_i|/d$ and each $g_i$ was chosen according to a standard
Gaussian distribution. We then let $X=UDV^\top$, where $U$ and $V$ are random
$d\times d$ and $n\times d$ orthogonal matrices. This results in a data
matrix $X$ whose spectrum is the same as $D$.} $\lambda$. For comparison, we
also implemented Oja's algorithm, using several different step sizes, as well
as power iterations\footnote{We note that more sophisticated iterative
algorithms, such as the Lanczos method, can attain better performance than
power iterations. However, they are not directly comparable to power
iterations and VR-PCA, since they are inherently more complex and can require
considerably more memory.}. All algorithms were initialized from the same
random vector, chosen uniformly at random from the unit ball. Note that
compared to our analysis, this makes things harder for our algorithm, since
we require it to perform well also in the `burn-in' phase. The results are
displayed in figure \ref{fig:synthetic}, and we see that for all values of
$\lambda$ considered, VR-PCA converges much faster than all versions of Oja's
algorithm, on which it is based, as well as power iterations, even though we
did not tune its parameters. Moreover, since the $y$-axis is in logarithmic
scale, we see that the convergence rate is indeed exponential in general,
which accords with our theory. In contrast, the convergence rate of Oja's
algorithm (no matter which step size is chosen) appears to be
sub-exponential. This is not surprising, since the algorithm does not
leverage the finite nature of the training data, and the inherent variance in
its updates does not decrease exponentially fast. A similar behavior will
occur with other purely stochastic algorithms in the literature, such as
\cite{ACLS12,mitliagkas2013memory,hardt2014noisy,de2014global}.

Next, we performed a similar experiment using the training data of the
well-known MNIST and CCAT datasets. The MNIST data matrix size is $784\times
70000$, and was pre-processed by centering the data and dividing each
coordinate by its standard deviation times the squared root of the dimension.
The CCAT data matrix is sparse (only 0.16\% of entries are non-zero), of size
$23149\times 781265$, and was used as-is. The results appear in figure
\ref{fig:real}. We also present the results for a simple hybrid method, which
initializes the VR-PCA algorithm with the result of running $n$ iterations of
Oja's algorithm. The decaying step size of Oja's algorithm is more suitable
for the initial phase, and the resulting hybrid algorithm can perform better
than each algorithm alone.

Finally, we present a similar experiment on the MNIST and CCAT datasets,
where this time we attempt to recover $k>1$ singular vectors using the
generalization of VR-PCA discussed in remark \ref{remark:k1}. A similar
generalization was also employed with the competitors. The results are
displayed in figure \ref{fig:realblock}, and are qualitatively similar to the
$k=1$ case.

\section{Proof of \thmref{thm:main}}\label{sec:proof}

To simplify the presentation of the proof, we use a few important
conventions:
\begin{itemize}
\item Note that the algorithm remains the same if we divide each $\bx_i$ by
    $\sqrt{r}$, and multiply $\eta$ by $r$. Since $\max_i
    \norm{\bx_i}^2\leq r$, this corresponds to running the algorithm with
    step-size $\eta r$ rather than $\eta$, on a re-scaled dataset of points
    with squared norm at most $1$, and with an eigengap of $\lambda/r$
    instead of $\lambda$. Therefore, we can simply analyze the algorithm
    assuming that $\max_i \norm{\bx_i}^2\leq 1$, and in the end plug in
    $\lambda/r$ instead of $\lambda$, and $\eta r$ instead of $\eta$, to
    get a result which holds for data with squared norm at most $r$.
\item Let $ A= \sum_{i=1}^{d}s_i \bv_i\bv_i^\top $ be an eigendecomposition
    of $A$, where $s_1>s_2\geq\ldots \geq s_d$, $s_1-s_2=\lambda>0$, and
    $\bv_1,\ldots,\bv_d$ are orthonormal vectors. Following the discussion
    above, we assume that $\max_i \norm{\bx_i}^2\leq 1$ and therefore
    $\{s_1,\ldots,s_d\}\subset [0,1]$.
\item Throughout the proof, we use $c$ to designate positive numerical
    constants, whose value can vary at different places (even in the same
    line or expression).
\end{itemize}

\subsection*{Part I: Establishing a Stochastic Recurrence Relation}

We begin by focusing on a single epoch of the algorithm, and a single
iteration $t$, and analyze how $1-\inner{\bw_t,\bv_1}^2$ evolves during that
iteration. The key result we need is the following lemma:

\begin{lemma}\label{lem:recur}
Suppose that $\inner{\bw_t,\bv_1}\geq \frac{1}{2}$, and that
$\inner{\tilde{\bw}_{s-1},\bv_1}\geq 0$. If $\eta\leq
c\lambda$, then
\[
\E\left[\left(1-\inner{\bw_{t+1},\bv_1}^2\right)\middle| \bw_t\right] ~\leq~
\left(1-\frac{\eta\lambda}{16}\right)\left(1-\inner{\bw_t,\bv_1}^2\right)
+c\eta^2\left(1-\inner{\tilde{\bw}_{s-1},\bv_1}^2\right)
\]
for certain positive numerical constants $c$.
\end{lemma}

\begin{proof}
Since we focus on a particular epoch $s$, let us drop the subscript from
$\tilde{\bw}_{s-1}$, and denote it simply at $\tilde{\bw}$. Rewriting the
update equations from the algorithm, we have that
\[
\bw_{t+1}=\frac{\bw'_{t+1}}{\norm{\bw'_{t+1}}}~,~\text{where}~~~
\bw'_{t+1}=(I+\eta A)\bw_t+\eta(\bx\bx^\top-A)(\bw_t-\tilde{\bw}),
\]
where $\bx$ is the random instance chosen at iteration $t$.

It is easy to verify that
\begin{equation}\label{eq:az}
\inner{\bw'_{t+1},\bv_i} = a_i+z_i,
\end{equation}
where
\[
a_i = (1+\eta s_i)\inner{\bw_t,\bv_i}~~,~~
z_i = \eta \bv_i^\top(\bx\bx^\top-A)(\bw_t-\tilde{\bw}).
\]
Moreover, since $\bv_1,\ldots,\bv_d$ form an orthonormal basis in $\reals^d$,
we have
\begin{equation}\label{eq:az2}
\norm{\bw'_{t+1}}^2 = \sum_{i=1}^{d}\inner{\bv_i,\bw'_{t+1}}^2 = \sum_{i=1}^{d}(a_i+z_i)^2.
\end{equation}

Let $\E$ denote expectation with respect to $\bx$, conditioned on $\bw_t$.
Combining \eqref{eq:az} and \eqref{eq:az2}, we have
\begin{equation}
  \E\left[\inner{\bw_{t+1},\bv_1}^2\right] =
  \E\left[\inner{\frac{\bw'_{t+1}}{\norm{\bw'_{t+1}}},\bv_1}^2\right] =
  \E\left[\frac{\inner{\bw'_{t+1},\bv_1}^2}{\norm{\bw'_{t+1}}^2}\right] =
  \E\left[\frac{(a_1+z_1)^2}{\sum_{i=1}^{d}(a_i+z_i)^2}\right].\label{eq:az3}
\end{equation}
Note that conditioned on $\bw_t$, the quantities $a_1\ldots a_d$ are fixed,
whereas $z_1\ldots z_d$ are random variables (depending on the random choice
of $\bx$) over which we take an expectation.

The first step of the proof is to simplify \eqref{eq:az3}, by pushing the
expectations inside the numerator and the denominator. Of course, this may
change the value of the expression, so we need to account for this change
with some care. To do so, define the auxiliary non-negative random variables
$x,y$ and a function $f(x,y)$ as follows:
\[
x = (a_1+z_1)^2~~,~~ y = \sum_{i=2}^{d}(a_i+z_i)^2~~,~~f(x,y) = \frac{x}{x+y}.
\]
Then we can write \eqref{eq:az3} as $\E_{x,y}[f(x,y)]$. We now use a
second-order Taylor expansion to relate it to
$f(\E[x],\E[y])=\frac{\E[(a_1+z_1)^2]}{\E[\sum_{i=1}^{d}(a_i+z_i)^2]}$.
Specifically, we have that $\E_{x,y}[f(x,y)]$ can be lower bounded by
\begin{align}
\E_{x,y} &\left[f(\E[x],\E[y])+\nabla f(\E[x],\E[y])^\top\left({x\choose y}-{\E[x]\choose \E[y]}\right)-\max_{x,y}\norm{\nabla^2 f(x,y)}\max_{x,y}\left\|{x\choose y}-{\E[x]\choose \E[y]}\right\|^2\right]\notag\\
&=f(\E[x],\E[y])-\max_{x,y}\norm{\nabla^2 f(x,y)}\max_{x,y}\left\|\left(\begin{array}{c}x-\E[x]\\y-\E[y]\end{array}\right)\right\|^2,\label{eq:taylor}
\end{align}
where $\nabla^2 f(x,y)$ is the Hessian of $f$ at $(x,y)$.

We now upper bound the two max-terms in the expression above:
\begin{itemize}
\item For the first max-term, it is easily verified that
\[
\nabla^2 f(x,y) = \frac{1}{(x+y)^3}\left(
                    \begin{array}{cc}
                      -2y & x-y \\
                      x-y & 2x \\
                    \end{array}
                  \right).
\]
Since the spectral norm is upper bounded by the Frobenius norm, which for
$2\times 2$ matrices is upper bounded by $2$ times the magnitude of the
largest entry in the matrix (which in our case is at most
$2(x+y)/(x+y)^3=2/(x+y)^2 \leq 2/x^2$), we have
\[
\max_{x,y}\norm{\nabla^2 f(x,y)}\leq \max_{x} \frac{4}{x^2}=\max_{z_1}\frac{4}{(a_1+z_1)^2}.
\]
Now, recall that $a_1\geq \frac{1}{2}$ by the Lemma's assumptions, and in
contrast $|z_1|\leq \eta
\left|\bv_i^\top(\bx\bx^\top-A)(\bw_t-\tilde{\bw})\right|\leq
\eta\norm{\bv_i}\norm{\bx\bx^\top-A}\norm{\bw_t-\tilde{\bw}}\leq c\eta$, so
for $\eta$ sufficiently small, $|z_1|\leq \frac{1}{2}|a_1|$, and we can
upper bound $\frac{4}{(a_1+z_1)^2}$ (and hence $\max_{x,y}\norm{\nabla^2
f(x,y)}$) by some numerical constant $c$. Overall, we have
\begin{equation}\label{eq:tmaz0}
\max_{x,y}\norm{\nabla^2
f(x,y)}\leq c.
\end{equation}

\item For the second max-term in \eqref{eq:taylor}, recalling that
    $x=(a_1+z_1)^2$, $y=\sum_{i=2}^{d}(a_i+z_i)^2$, and that the $z_i$'s
    are zero-mean, we have
\begin{align*}
  \max_{x,y}\left((x-\E[x])^2+(y-\E[y])^2\right)~&=~
  \max_{z_1\ldots z_d}\left(2a_1 z_1+z_1^2-\E[z_1^2]\right)^2+\left(\sum_{i=2}^{d}\left(2a_i z_i+z_i^2-\E[z_i^2]\right)\right)^2.
\end{align*}
Using the elementary fact that $(r+s)^2\leq 2(r^2+s^2)$ for all $r,s$, as
well as the definition of $a_i,z_i$, this is at most
\begin{align}
&2\max_{z_1}(2a_1z_1)^2 +2\max_{z_1}(z_1^2-\E[z_1^2])^2+2\max_{z_2,\ldots,z_d}\left(\sum_{i=2}^{d}2a_i z_i\right)^2+
  2\max_{z_2,\ldots,z_d}\left(\sum_{i=2}^{d}(z_i^2-\E[z_i^2])\right)^2\notag\\
&\leq 8\max_{z_1}(a_1z_1)^2+4\max_{z_1}(z_1^2)^2+4\E[z_1^2]^2+8\max_{z_2,\ldots,z_d}\left(\sum_{i=2}^{d}a_i z_i\right)^2+4\max_{z_2,\ldots,z_d}\left(\sum_{i=2}^{d}z_i^2\right)+4\left(\sum_{i=2}^{d}\E[z_i^2]\right)^2\notag\\
&\leq
8\max_{z_1}(a_1z_1)^2+8\max_{z_1}z_1^4+8\max_{z_2,\ldots,z_d}\left(\sum_{i=2}^{d}a_i
z_i\right)^2+8\max_{z_2,\ldots,z_d}\left(\sum_{i=2}^{d}z_i^2\right)^2.
\label{eq:tmaz1}
\end{align}
Recalling the definition of $a_i,z_i$, and that
$\norm{\bw_t}$,$\norm{\tilde{\bw}}$,$\norm{\bv_1}$,$\eta s_i$ and
$\norm{\bx\bx^\top-A}$ are all bounded by constants, we now show that each
term in the expression above can be upper bounded by
$c\eta^2\norm{\bw_t-\tilde{\bw}}^2$ for some appropriate constant $c$:
\begin{align*}
  8\max_{z_1}(a_1z_1)^2&= 8\max_{\bx}\left(\eta(1+\eta s_1)\inner{\bw_t,\bv_1}\bv_1^\top(\bx\bx^\top-A)(\bw_t-\tilde{\bw})\right)^2\\
  &\leq 8\max_{\bx}\left(\eta(1+\eta s_1)|\inner{\bw_t,\bv_1}|\norm{\bv_1}\norm{\bx\bx^\top-A}\norm{\bw_t-\tilde{\bw}}\right)^2\\
  &\leq c\left(\eta\norm{\bw_t-\tilde{\bw}}\right)^2~=~c\eta^2\norm{\bw_t-\tilde{\bw}}^2.
\end{align*}
\begin{align*}
  8\max_{z_1}z_1^4&= 8\max_{\bx}\left(\eta\bv_1^\top(\bx\bx^\top-A)(\bw_t-\tilde{\bw})\right)^4\\
  &\leq 8\left(\eta\norm{\bv_1}\norm{\bx\bx^\top-A}\norm{\bw_t-\tilde{\bw}}\right)^4\\
  &\leq c(\eta\norm{\bw_t-\tilde{\bw}})^4~=~ c(\eta\norm{\bw_t-\tilde{\bw}})^2(\eta\norm{\bw_t-\tilde{\bw}})^2\\
  &\leq c(\eta\norm{\bw_t-\tilde{\bw}})^2 = c\eta^2\norm{\bw_t-\tilde{\bw}}^2.
\end{align*}
\begin{align*}
  8\max_{z_2,\ldots,z_d}\left(\sum_{i=2}^{d}a_i z_i\right)^2&=
  8\max_{\bx}\left(\sum_{i=2}^{d}\eta(1+\eta s_i)\inner{\bw_t,\bv_i}\bv_i^\top(\bx\bx^\top-A)(\bw_t-\tilde{\bw})\right)^2\\
  &\leq c\left(\eta\left\|\sum_{i=2}^{d}(1+\eta s_i)\inner{\bw_t,\bv_i}\bv_i\right\|\norm{\bw_t-\tilde{\bw}}\right)^2\\
  &= c\left(\eta\left\|\left(\sum_{i=2}^{d}(1+\eta s_i)\bv_i\bv_i^\top\right) \bw_t\right\|\norm{\bw_t-\tilde{\bw}}\right)^2\\
  &\leq c\left(\eta\left\|\left(\sum_{i=2}^{d}(1+\eta s_i)\bv_i\bv_i^\top\right)\right\|\norm{\bw_t-\tilde{\bw}}\right)^2\\
  &\leq c\left(\eta\norm{\bw_t-\tilde{\bw}}\right)^2~=~ c\eta^2\norm{\bw_t-\tilde{\bw}}^2,
\end{align*}
where in the last inequality we used the fact that $\bv_2\ldots \bv_d$ are
orthonormal vectors, and $(1+\eta s_i)$ is bounded by a constant. Similarly,
\begin{align*}
  8\max_{z_2,\ldots,z_d}\left(\sum_{i=2}^{d}z_i^2\right)^2&=
  8\max_{\bx}\left(\eta^2\sum_{i=2}^{d}(\bw_t-\tilde{\bw})^\top (\bx\bx^\top-A)\bv_i\bv_i^\top (\bx\bx^\top-A)(\bw_t-\tilde{\bw})\right)^2\\
  &=  8\max_{\bx}\left(\eta^2(\bw_t-\tilde{\bw})^\top (\bx\bx^\top-A)\left(\sum_{i=2}^{d}\bv_i\bv_i^\top\right) (\bx\bx^\top-A)(\bw_t-\tilde{\bw})\right)^2\\
  &\leq 8\max_{\bx}\left(\eta^2\norm{\bw_t-\tilde{\bw}}^2\norm{\bx\bx^\top-A}^2\left\|\sum_{i=2}^{d}\bv_i\bv_i^\top\right\|\right)^2\\
  &\leq c\left(\eta^2\norm{\bw_t-\tilde{\bw}}^2\right)^2 ~=~ c\left(\eta^2\norm{\bw_t-\tilde{\bw}}^2\right)\left(\eta^2\norm{\bw_t-\tilde{\bw}}^2\right)
  ~\leq~ c\eta^2\norm{\bw_t-\tilde{\bw}}^2.
\end{align*}
Plugging these bounds back into \eqref{eq:tmaz1}, we get that
\begin{equation}\label{eq:tmaz05}
\max_{x,y}\left((x-\E[x])^2+(y-\E[y])^2\right)\leq
c\eta^2\norm{\bw_t-\tilde{\bw}}^2
\end{equation}
for some appropriate constant $c$.
\end{itemize}

Plugging \eqref{eq:tmaz0} and \eqref{eq:tmaz1} back into \eqref{eq:taylor},
we get a lower bound of
\[
\E_{x,y}[f(x,y)]\geq f(\E[x],\E[y])-c\eta^2\norm{\bw_t-\tilde{\bw}}^2
= \frac{\E\left[(a_1+z_1)^2\right]}{\E\left[\sum_{i=1}^{d}(a_i+z_i)^2\right]}-c\eta^2\norm{\bw_t-\tilde{\bw}}^2,
\]
and since each $z_i$ is zero-mean, this equals
\begin{equation}\label{eq:interbound}
\frac{\E\left[a_1^2+z_1^2\right]}{\E\left[\sum_{i=1}^{d}(a_i^2+z_i^2)\right]}-c\eta^2\norm{\bw_t-\tilde{\bw}}^2
\end{equation}
By definition of $z_i$ and the fact that $\bv_1,\ldots,\bv_d$ are orthonormal
(hence $\sum_i \bv_i\bv_i^\top$ is the identity matrix), we have
\begin{align*}
\sum_{i=1}^{d}z_i^2&=
\eta^2(\bw_t-\tilde{\bw})^\top(\bx\bx^\top-A)\left(\sum_{i=1}^{d}\bv_i\bv_i^\top\right)(\bx\bx^\top-A)(\bw_t-\tilde{\bw})\\
&=
\eta^2(\bw_t-\tilde{\bw})^\top(\bx\bx^\top-A)(\bx\bx^\top-A)(\bw_t-\tilde{\bw})\\
&=
\eta^2\norm{(\bx\bx^\top-A)(\bw_t-\tilde{\bw})}^2~\leq~ c\eta^2\norm{\bw_t-\tilde{\bw}}^2,
\end{align*}
so we can lower bound \eqref{eq:interbound} by
\begin{equation}\label{eq:interbound2}
\frac{a_1^2}{\sum_{i=1}^{d}a_i^2+c\eta^2\norm{\bw_t-\tilde{\bw}}^2}-c\eta^2\norm{\bw_t-\tilde{\bw}}^2.
\end{equation}
Focusing on the first term in \eqref{eq:interbound2} for the moment, and
substituting in the definition of $a_i$, we can write it as
\begin{align*}
&\frac{(1+\eta s_1)^2\inner{\bw_t,\bv_1}^2}
{(1+\eta s_1)^2\inner{\bw_t,\bv_1}^2+\sum_{i=2}^{d}(1+\eta s_i)^2\inner{\bv_i,\bw_t}^2+c\eta^2\norm{\bw_t-\tilde{\bw}}^2}\\
&\geq \frac{\inner{\bw_t,\bv_1}^2}
{\inner{\bw_t,\bv_1}^2+\left(\frac{1+\eta s_2}{1+\eta s_1}\right)^2\sum_{i=2}^{d}\inner{\bv_i,\bw_t}^2+c\eta^2\norm{\bw_t-\tilde{\bw}}^2}\\
&= \frac{\inner{\bw_t,\bv_1}^2}
{\inner{\bw_t,\bv_1}^2+\left(\frac{1+\eta s_2}{1+\eta s_1}\right)^2\left(1-\inner{\bw_t,\bv_1}^2\right)+c\eta^2\norm{\bw_t-\tilde{\bw}}^2}\\
&=\frac{\inner{\bw_t,\bv_1}^2}
{1-\left(1-\left(\frac{1+\eta s_2}{1+\eta s_1}\right)^2\right)\left(1-\inner{\bw_t,\bv_1}^2\right)+c\eta^2\norm{\bw_t-\tilde{\bw}}^2}\\
&\geq \inner{\bw_t,\bv_1}^2\left(1+\left(1-\left(\frac{1+\eta s_2}{1+\eta s_1}\right)^2\right)\left(1-\inner{\bw_t,\bv_1}^2\right)-c\eta^2\norm{\bw_t-\tilde{\bw}}^2\right),
\end{align*}
where in the last step we used the elementary inequality $\frac{1}{1-x}\geq
1+x$ for all $x\leq 1$ (and this is indeed justified since
$\inner{\bw_t,\bv_1}\leq 1$ and $\frac{1+\eta s_2}{1+\eta s_1}\leq 1$). This
can be further lower bounded by
\begin{align*}
&\inner{\bw_t,\bv_1}^2\left(1+\left(1-\left(\frac{1+\eta s_2}{1+\eta s_1}\right)\right)\left(1-\inner{\bw_t,\bv_1}^2\right)-c\eta^2\norm{\bw_t-\tilde{\bw}}^2\right)\\
&= \inner{\bw_t,\bv_1}^2\left(1+\frac{\eta(s_1-s_2)}{1+\eta s_1}\left(1-\inner{\bw_t,\bv_1}^2\right)-c\eta^2\norm{\bw_t-\tilde{\bw}}^2\right)\\
&\geq \inner{\bw_t,\bv_1}^2\left(1+\frac{\eta\lambda}{2}
\left(1-\inner{\bw_t,\bv_1}^2\right)-c\eta^2\norm{\bw_t-\tilde{\bw}}^2\right),
\end{align*}
where in the last inequality we used the fact that $s_1-s_2=\lambda$ and that $\eta s_1\leq \eta$
which is at most $1$ (again using the assumption that $\eta$ is sufficiently small).

Plugging this lower bound on the first term in \eqref{eq:interbound2}, and
recalling that $\inner{\bw_t,\bv_1}^2$ is assumed to be at least $1/4$, we
get the following lower bound on \eqref{eq:interbound2}:
\begin{align*}
&\inner{\bw_t,\bv_1}^2\left(1+\frac{\eta\lambda}{2}\left(1-\inner{\bw_t,\bv_1}^2\right)-c\eta^2\norm{\bw_t-\tilde{\bw}}^2\right)
-c\eta^2\norm{\bw_t-\tilde{\bw}}^2\\
&\geq~\inner{\bw_t,\bv_1}^2\left(1+\frac{\eta\lambda}{2}\left(1-\inner{\bw_t,\bv_1}^2\right)-c\eta^2\norm{\bw_t-\tilde{\bw}}^2\right).
\end{align*}

To summarize the derivation so far, starting from \eqref{eq:az3} and
concatenating the successive lower bounds we have derived, we get that
\begin{equation}\label{eq:az8}
\E[\inner{\bw_{t+1},\bv_1}^2] \geq \inner{\bw_t,\bv_1}^2\left(1+\frac{\eta\lambda}{2}\left(1-\inner{\bw_t,\bv_1}^2\right)
-c\eta^2\norm{\bw_t-\tilde{\bw}}^2\right).
\end{equation}

We now get rid of the $\norm{\bw_t-\tilde{\bw}}^2$ term, by noting that since
$(x+y)^2\leq 2(x^2+y^2)$ and $\norm{\bw_t}=\norm{\bv_1}=1$,
\begin{align*}
\norm{\bw_t-\tilde{\bw}}^2&\leq \left(\norm{\bw_t-\bv_1}+\norm{\tilde{\bw}-\bv_1}\right)^2 \leq
2\left(\norm{\bw_t-\bv_1}^2+\norm{\tilde{\bw}-\bv_1}^2\right) \\&=
2\left(2-2\inner{\bw_t,\bv_1}+2-2\inner{\tilde{\bw},\bv_1}\right).
\end{align*}
Since we assume that $\inner{\bw_t,\bv_1},\inner{\tilde{\bw},\bv_1}$ are both
positive, and they are also at most $1$, this is at most
\[
2\left(2-2\inner{\bw_t,\bv_1}^2+2-2\inner{\tilde{\bw},\bv_1}^2\right)=4\left(1-\inner{\bw_t,\bv_1}^2\right)+4\left(1-\inner{\tilde{\bw},\bv_1}^2\right).
\]
Plugging this back into \eqref{eq:az8}, we get that
\[
\E[\inner{\bw_{t+1},\bv_1}^2] \geq \inner{\bw_t,\bv_1}^2\left(1+\left(\frac{\eta\lambda}{2}-c\eta^2\right)\left(1-\inner{\bw_t,\bv_1}^2\right)
-c\eta^2\left(1-\inner{\tilde{\bw},\bv_1}^2\right)\right),
\]
and since we can assume $\frac{\eta\lambda}{2}-c\eta^2 \geq
\frac{\eta\lambda}{4}$ by picking $\eta$ sufficiently smaller than $\lambda$,
this can be simplified to
\[
\E[\inner{\bw_{t+1},\bv_1}^2] \geq \inner{\bw_t,\bv_1}^2\left(1+\frac{\eta\lambda}{4}\left(1-\inner{\bw_t,\bv_1}^2\right)
-c\eta^2\left(1-\inner{\tilde{\bw},\bv_1}^2\right)\right).
\]
Finally, subtracting both sides of the inequality from $1$, we get
\begin{align*}
  \E[1-\inner{\bw_{t+1},\bv_1}^2] &\leq 1-\inner{\bw_t,\bv_1}^2-\frac{\eta\lambda}{4}\inner{\bw_t,\bv_1}^2\left(1-\inner{\bw_t,\bv_1}^2\right)
  +c\eta^2\inner{\bw_t,\bv_1}^2\left(1-\inner{\tilde{\bw},\bv_1}^2\right)\\
  &\leq \left(1-\frac{\eta\lambda}{4}\inner{\bw_t,\bv_1}^2\right)\left(1-\inner{\bw_t,\bv_1}^2\right)+c\eta^2\left(1-\inner{\tilde{\bw},\bv_1}^2\right),
\end{align*}
and since we assume $\inner{\bw_t,\bv_1}\geq \frac{1}{2}$, we can upper bound
this by
\[
\left(1-\frac{\eta\lambda}{16}\right)
\left(1-\inner{\bw_t,\bv_1}^2\right)
+c\eta^2\left(1-\inner{\tilde{\bw},\bv_1}^2\right)
\]
as required. Note that to get this bound, we assumed at several places that
$\eta$ is smaller than either a constant, or a constant factor times
$\lambda$ (which is at most $1$). Hence, the bound holds by assuming
$\eta\leq c\lambda$ for a sufficiently small constant $c$.
\end{proof}

\subsection*{Part II: Solving the Recurrence Relation for a Single Epoch}

As before, since we focus on a single epoch, we drop the subscript from
$\tilde{\bw}_{s-1}$ and denote it simply as $\tilde{\bw}$.

Suppose that $\eta=\alpha\lambda$, where $\alpha$ is a sufficiently small
constant to be chosen later. Also, let
\[
b_t = 1-\inner{\bw_t,\bv_1}^2~~~\text{and}~~~ \tilde{b} = 1-\inner{\tilde{\bw},\bv_1}^2.
\]
Then \lemref{lem:recur} tells us that if $\alpha$ is sufficiently small,
$b_t\leq \frac{3}{4}$, and $\inner{\tilde{\bw},\bv_1}\geq 0$, then
\begin{equation}\label{eq:bform}
\E\left[b_{t+1}\middle| \bw_t\right] ~\leq~
\left(1-\frac{\alpha}{16}\lambda^2\right)b_t
+c\alpha^2\lambda^2\tilde{b}.
\end{equation}

\begin{lemma}\label{lem:recurse}
Let $B$ be the event that $b_t\leq \frac{3}{4}$ for all $t=0,1,2,\ldots,m$.
Then for certain positive numerical constants $c$, if $\alpha\leq c$, and
$\inner{\tilde{\bw},\bv_1}\geq 0$, then
\[
\E[b_{m}|B,\bw_0]\leq \left(\left(1-\frac{\alpha}{16}\lambda^2\right)^{m}+c\alpha\right) \tilde{b}.
\]
\end{lemma}
\begin{proof}
  Recall that $b_t$ is a deterministic function of the random variable
  $\bw_t$, which depends in turn on $\bw_{t-1}$ and the random instance chosen at round
  $m$. We assume that $\bw_0$ (and hence $\tilde{b}$) are fixed, and consider how $b_t$ evolves as a function of $t$. Using \eqref{eq:bform}, we have
  \begin{align*}
  \E[b_{t+1}|\bw_{t},B] = \E\left[b_{t+1}|\bw_t,b_{t+1}\leq \frac{3}{4}\right]
  ~\leq~ \E[b_{t+1}|\bw_t] ~\leq~ \left(1-\frac{\alpha}{16}\lambda^2\right)b_t
~+~c\alpha^2\lambda^2\tilde{b}.
\end{align*}
Note that the first equality holds, since conditioned on $\bw_t$, $b_{t+1}$
is independent of $b_1,\ldots,b_{t}$, so the event $B$ is equivalent to just
requiring $b_{t+1}\leq 3/4$.

Taking expectation over $\bw_t$ (conditioned on $B$), we get that
\begin{align*}
  \E[b_{t+1}|B] ~&\leq~ \E\left[\left(1-\frac{\alpha}{16}\lambda^2\right)b_{t}
+c\alpha^2\lambda^2\tilde{b}\middle| B\right]\\
&=\left(1-\frac{\alpha}{16}\lambda^2\right)\E\left[b_{t}|B\right]
+c\alpha^2\lambda^2\tilde{b}.
\end{align*}
Unwinding the recursion, and using that $b_0=\tilde{b}$, we therefore get
that
\begin{align*}
\E[b_{m}|B]~&\leq~\left(1-\frac{\alpha}{16}\lambda^2\right)^{m}\tilde{b}+c\alpha^2\lambda^2\tilde{b}\sum_{i=0}^{m-1}\left(1-\frac{\alpha}{16}\lambda^2\right)^i\\
&\leq~\left(1-\frac{\alpha}{16}\lambda^2\right)^{m}\tilde{b}+c\alpha^2\lambda^2\tilde{b}\sum_{i=0}^{\infty}\left(1-\frac{\alpha}{16}\lambda^2\right)^i\\
&=~\left(1-\frac{\alpha}{16}\lambda^2\right)^{m}\tilde{b}+c\alpha^2\lambda^2\tilde{b}\frac{1}{(\alpha/16)\lambda^2}\\
&=~\left(\left(1-\frac{\alpha}{16}\lambda^2\right)^{m}+c\alpha\right) \tilde{b}.\\
\end{align*}
as required.
\end{proof}

We now turn to prove that the event $B$ assumed in \lemref{lem:recurse}
indeed holds with high probability:
\begin{lemma}\label{lem:event}
  For certain positive numerical constants $c$,
  suppose that $\alpha\leq c$, and $\inner{\tilde{\bw},\bv_1}\geq 0$. Then for any $\beta\in (0,1)$ and $m$, if
  \begin{equation}\label{eq:event}
  \tilde{b}+cm\alpha^2\lambda^2 +c\sqrt{m\alpha^2\lambda^2\log(1/\beta)}\leq \frac{3}{4},
  \end{equation}
  for a certain numerical constant $c$, then it holds with probability at least $1-\beta$ that
  \[
  b_t~\leq~ \tilde{b}+cm\alpha^2\lambda^2 +c\sqrt{m\alpha^2\lambda^2\log(1/\beta)}~\leq~ \frac{3}{4}
  \]
  for some numerical constant $c$ and for all $t=0,1,2,\ldots,m$, as well as
  $\inner{\bw_m,\bv_1}\geq 0$.
\end{lemma}
\begin{proof}
 To prove the lemma, we analyze the stochastic process $b_0(=\tilde{b}),b_1,b_2,\ldots,b_m$, and
 use a concentration of measure argument. First, we collect the following
 facts:
 \begin{itemize}
     \item \emph{$\tilde{b}=b_0\leq \frac{3}{4}$}: This directly follows
         from the assumption stated in the lemma.
     \item \emph{The conditional expectation of $b_{t+1}$ is close to
         $b_t$, as long as $b_t\leq \frac{3}{4}$ }: Supposing that $b_t\leq
         \frac{3}{4}$ for some $t$, and $\alpha$ is sufficiently small,
         then by \eqref{eq:bform},
 \begin{align*}
 \E\left[b_{t+1}\middle| \bw_t\right] ~&\leq~
\left(1-\frac{\alpha}{16}\lambda^2\right)b_t
+c\alpha^2\lambda^2\tilde{b}
~\leq~ b_t+c\alpha^2\lambda^2 \tilde{b}.
 \end{align*}
    \item \emph{$|b_{t+1}-b_t|$ is bounded by $c\alpha\lambda$}: Since the
        norm of $\bw_t,\bv_1$ is $1$, we have
 \begin{align*}
 |b_{t+1}-b_{t}|~&=~
 \left|\inner{\bw_{t+1},\bv_1}^2-\inner{\bw_{t},\bv_1}^2\right|~=~
 \left|\inner{\bw_{t+1},\bv_1}+\inner{\bw_{t},\bv_1}\right|*\left|\inner{\bw_{t+1},\bv_1}-\inner{\bw_{t},\bv_1}\right|\\
 &\leq 2\left|\inner{\bw_{t+1},\bv_1}-\inner{\bw_{t},\bv_1}\right|~\leq~
 2\norm{\bw_{t+1}-\bw_t}.
 \end{align*}
  Recalling the definition of $\bw_{t+1}$ in our algorithm, and the fact
 that the instances $\bx_i$ and hence the matrix $A$ are assumed to have
 norm at most $1$, it is easy to verify that $\norm{\bw_{t+1}-\bw_t}\leq
 c\eta\leq c\alpha\lambda$ for some appropriate constant $c$.
 \end{itemize}
 Armed with these facts, and using the maximal version of the Hoeffding-Azuma inequality \cite{hoeffding1963probability}, it follows that with probability at least
 $1-\beta$, it holds simultaneously for all $t=1,\ldots,m$ (and for $t=0$ by assumption) that
 \[
 b_t\leq \tilde{b}+mc\alpha^2\lambda^2 \tilde{b}+c\sqrt{m\alpha^2\lambda^2\log(1/\beta)}
 \]
 for some constants $c$, as long as the expression above is less than
 $\frac{3}{4}$. If the expression is indeed less than $\frac{3}{4}$, then we
 get that $b_t\leq \frac{3}{4}$ for all $t$. Upper bounding $\tilde{b}$ and $\lambda$ by $1$, and slightly simplifying, we get the
 statement in the lemma.

 It remains to prove that if $b_t\leq \frac{3}{4}$ for all $t$, then
 $\inner{\bw_m,\bv_1}\geq 0$. Suppose on the
 contrary that $\inner{\bw_m,\bv_1}<0$. Since
 $|\inner{\bw_{t+1},\bv_1}-\inner{\bw_{t},\bv_1}|\leq
 \norm{\bw_{t+1}-\bw_{t}}\leq c\alpha\lambda$ as we've seen earlier, and $\inner{\bw_0,\bv_1}\geq 0$, it means there must have been some $\bw_t$ such that
 $\inner{\bw_t,\bv_1}\leq c\alpha\lambda$. But this means that
 $b_t=(1-\inner{\bw_t,\bv_1}^2)\geq 1-c^2\alpha^2\lambda^2 >
 \frac{3}{4}$ (as long as $\alpha$ is sufficiently small, since we assume $\lambda$ is bounded), invalidating the assumption that $b_t\leq \frac{3}{4}$ for all $t$. Therefore, $\inner{\bw_m,\bv_1}\geq 0$ as required.
\end{proof}

Combining \lemref{lem:recurse} and \lemref{lem:event}, and using Markov's
inequality, we get the following corollary:

\begin{lemma}\label{lem:combine}
Let confidence parameters $\beta,\gamma\in(0,1)$ be fixed. Suppose that
$\inner{\tilde{\bw},\bv_1}\geq 0$, and that $m,\alpha$ are chosen such that
\[
\tilde{b}+cm\alpha^2\lambda^2+c\sqrt{m\alpha^2\lambda^2\log(1/\beta)}\leq \frac{3}{4}
\]
for a certain numerical constant $c$. Then with probability at least
$1-(\beta+\gamma)$, it holds that $\inner{\bw_m,\bv_1}\geq 0$, and
\[
b_m \leq \frac{1}{\gamma}
\left(\left(1-\frac{\alpha}{16}\lambda^2\right)^{m}+c\alpha\right) \tilde{b}.
\]
for some numerical constant $c$.
\end{lemma}%

\subsection*{Part III: Analyzing the Entire Algorithm's Run}

Given the analysis in \lemref{lem:combine} for a single epoch, we are now
ready to prove our theorem. Let
\[
\tilde{b}_s = 1-\inner{\tilde{\bw}_s,\bv_1}^2.
\]
By assumption, at the beginning of the first epoch, we have
$\tilde{b}_0=1-\inner{\tilde{\bw}_0,\bv_1}^2\leq 1-\frac{1}{2}=\frac{1}{2}$.
Therefore, by \lemref{lem:combine}, for any
$\beta,\gamma\in\left(0,\frac{1}{2}\right)$, if we pick any
\begin{equation}\label{eq:condme}
\alpha\leq \frac{1}{2}\gamma^2~~~~\text{and}~~~~
m\geq \frac{48\log(1/\gamma)}{\alpha\lambda^2}
~~~~\text{such that}~~~~\frac{1}{2}+cm\alpha^2\lambda^2+c\sqrt{m\alpha^2\lambda^2\log(1/\beta)}\leq \frac{3}{4},
\end{equation}
then we get with probability at least $1-(\beta+\gamma)$ that
\[
\tilde{b}_1~\leq~
\frac{1}{\gamma}\left(\left(1-\frac{\alpha\lambda^2}{16}\right)^{\frac{48\log(1/\gamma)}{\alpha\lambda^2}}
+\frac{1}{2}\gamma^2\right)\tilde{b}_0
\]
Using the inequality $(1-(1/x))^{ax}\leq \exp(-a)$, which holds for any $x>1$
and any $a$, and taking $x = 16/(\alpha\lambda^2)$ and $a = 3\log(1/\gamma)$,
we can upper bound the above by
\begin{align*}
&\frac{1}{\gamma}\left(\exp\left(-3\log\left(\frac{1}{\gamma}\right)\right)+\frac{1}{2}\gamma^2\right)\tilde{b}_0\\
&=~ \frac{1}{\gamma}\left(\gamma^3+\frac{1}{2}\gamma^2\right)\tilde{b}_0 ~\leq~ \gamma\tilde{b}_0.
\end{align*}
Therefore, we get that $\tilde{b}_1\leq \gamma\tilde{b}_0$. Moreover, again
by \lemref{lem:combine}, we have $\inner{\tilde{\bw}_1,\bv_1}\geq 0$. Since
$\tilde{b}_1$ is only smaller than $\tilde{b}_0$, the conditions of
\lemref{lem:combine} are fulfilled for $\tilde{b}=\tilde{b}_1$, so again with
probability at least $1-(\beta+\gamma)$, by the same calculation, we have
\[
\tilde{b}_2~\leq~ \gamma
\tilde{b}_1~\leq~ \gamma^2\tilde{b}_0.
\]
Repeatedly applying \lemref{lem:combine} and using a union bound, we get that
after $T$ epochs, with probability at least $1-T(\beta+\gamma)$,
\[
1-\inner{\tilde{\bw}_T,\bv_1}^2~=~\tilde{b}_T ~\leq~ \gamma^T\tilde{b}_0 ~<~ \gamma^T.
\]
Therefore, for any desired accuracy parameter $\epsilon$, we simply need to
use $T=\left\lceil\frac{\log(1/\epsilon)}{\log(1/\gamma)}\right\rceil$
epochs, and get $1-\inner{\tilde{\bw}_T,\bv_1}^2\leq \epsilon$ with
probability at least
$1-T(\beta+\gamma)=1-\left\lceil\frac{\log(1/\epsilon)}{\log(1/\gamma)}\right\rceil(\beta+\gamma)$.

Using a confidence parameter $\delta$, we pick
$\beta=\gamma=\frac{\delta}{2}$, which ensures that the accuracy bound above
holds with probability at least
\[
1-\left\lceil\frac{\log(1/\epsilon)}{\log(2/\delta)}\right\rceil\delta
~\geq~
1-\frac{\log(1/\epsilon)}{\log(2/\delta)}\delta
~\geq~
1-2\log\left(\frac{1}{\epsilon}\right)\delta.
\]
Substituting this choice of $\beta,\gamma$ into \eqref{eq:condme}, and
recalling that the step size $\eta$ equals $\alpha\lambda$, we get that
$\inner{\tilde{\bw}_T,\bv_1}^2\geq 1-\epsilon$ with probability at least
$1-2\log(1/\epsilon)\delta$, provided that
\[
\eta \leq c\delta^2\lambda~~~~,~~~~
m\geq \frac{c\log(2/\delta)}{\eta \lambda}
~~~~,~~~~m\eta^2+\sqrt{m\eta^2\log(2/\delta)}\leq c
\]
for suitable constants $c$.

To get the theorem statement, recall that this analysis pertains to data
whose squared norm is bounded by $1$. By the reduction discussed at the
beginning of the proof,
 we can apply it to data with squared norm at most $r$, by replacing $\lambda$ with $\lambda/r$,
 and $\eta$ with $\eta r$, leading to the condition
\[
\eta \leq \frac{c\delta^2}{r^2}\lambda~~~~,~~~~
m\geq \frac{c\log(2/\delta)}{\eta \lambda}
~~~~,~~~~m\eta^2r^2+r\sqrt{m\eta^2\log(2/\delta)}\leq c.
\]
Recalling that the different $c$'s above correspond to possibly different
positive numerical constants, we get the result stated in the theorem.

%

\section{Discussion}

In this paper, we presented and analyzed a stochastic algorithm for PCA and
SVD with an exponential convergence rate. Under suitable assumptions, the
runtime scales as the \emph{sum} of the data size $n$ and an eigengap factor
$\frac{1}{\lambda^2}$, and \emph{logarithmically} in the required accuracy
$\epsilon$. In contrast, the runtime of previous iterative methods scale
either as the \emph{product} of $n$ and an eigengap factor, or
\emph{polynomially} in $\epsilon$.

This work leaves several open questions. First, we note that in the regime of
moderate data size $n$ (in particular, when $n$ is dominated by
$(r/\lambda)^2$), the required runtime scales with $1/\lambda^2$, which is
inferior to the deterministic methods discussed in \secref{sec:introduction}.
Second, in the context of strongly convex optimization problems, the
variance-reduced technique we use leads to algorithms with runtime $
\Ocal\left(d\left(n+\frac{1}{\lambda}\right)\log\left(\frac{1}{\epsilon}\right)\right),
$ where $\lambda$ is the strong convexity parameter of the problem
\cite{johnson2013accelerating}. Comparing this with our algorithm's runtime,
and drawing a parallel between strong convexity and the eigengap in PCA
problems, it is tempting to conjecture that the $1/\lambda^2$ in our runtime
analysis can be improved at least to $1/\lambda$. However, we don't know if
this is true, or whether the $1/\lambda^2$ factor is necessary in our
setting. Third, it remains to analyze the behavior of the algorithm starting
from a randomly initialized point, before we obtain some $\tilde{\bw}_0$
sufficiently close to the optimum. Experimentally, this does not seem to be
an issue, but a full analysis would be more satisfactory, and might give more
guidance on how to optimally choose the step size. Finally, we believe our
formal analysis should be extendable to the $k>1$ case (see remark
\ref{remark:k1}), and that the dependence on the maximal squared norm of the
data can be relaxed to a dependence on the average squared norm or some
weaker moment conditions.

\subsubsection*{Acknowledgments}
This research is supported in part by an FP7 Marie Curie CIG grant, the Intel
ICRI-CI Institute, and Israel Science Foundation grant 425/13. We thank Huy
Nguyen, Mingyi Hong and Haishan Ye for spotting a bug in the proof of lemma 1 in an earlier version of this paper.

\bibliographystyle{plain}
\bibliography{mybib}

\appendix

\section{Implementing Epochs in $\Ocal(d_s(m+n))$ Amortized
Runtime}\label{app:sparse}

As discussed in remark \ref{remark:sparse}, the runtime of each iteration in
our algorithm (as presented in our pseudo-code) is $\Ocal(d)$, and the total
runtime of each epoch is $\Ocal(dm+d_s n)$, where $d_s$ is the average
sparsity (number of non-zero entries) in the data points $\bx_i$. Here, we
explain how the total epoch runtime can be improved (at least in terms of the
theoretical analysis) to $\Ocal(d_s(m+n))$. For ease of exposition, we
reproduce the pseudo-code together with line numbers below:

\begin{center}
\begin{minipage}{0.7\textwidth}
\begin{algorithmic}[1]
\STATE \textbf{Parameters:} Step size $\eta$, epoch length $m$ \STATE
\textbf{Input:} Data matrix $X=(\bx_1,\ldots,\bx_n)$; Initial unit vector
$\tilde{\bw}_0$ \FOR{$s=1,2,\ldots$}
  \STATE $\tilde{\bu}=\frac{1}{n}\sum_{i=1}^{n}\bx_i\left(\bx_i^\top \tilde{\bw}_{s-1}\right)$
  \STATE $\bw_0=\tilde{\bw}_{s-1}$
  \FOR{$t=1,2,\ldots,m$}
    \STATE Pick $i_t\in \{1,\ldots,n\}$ uniformly at random
    \STATE $\bw'_{t}=\bw_{t-1}+\eta\left(\bx_{i_t}\left(\bx_{i_t}^\top\bw_{t-1}-\bx_{i_t}^\top\tilde{\bw}_{s-1}\right)+\tilde{\bu}\right)$
    \STATE $\bw_{t}=\frac{1}{\norm{\bw'_t}}\bw'_{t}$
  \ENDFOR
  \STATE $\tilde{\bw}_{s}=\bw_m$
\ENDFOR
\end{algorithmic}
\end{minipage}
\end{center}

First, we can assume without loss of generality that $d \leq d_s n$.
Otherwise, the number of non-zeros in the $n\times d$ data matrix $X$ is
smaller than $d$, so the matrix must contain some all-zeros columns. But
then, we can simply drop those columns (the value of the largest singular
vectors in the corresponding entries will be zero anyway), hence reducing the
effective dimension $d$ to be at most $d_s n$. Therefore, given a vector
$\tilde{\bw}_{s-1}$, we can implement line (4) in $\Ocal(d+d_s n)\leq
\Ocal(d_s n)$ time, by initializing the $d$-dimensional vector $\tilde{\bu}$
to be $0$, and iteratively adding to it the sparse (on-average) vector
$\bx_i\left(\bx_i^\top \tilde{\bw}_{s-1}\right)$. Similarly, we can implement
lines (5),(11) in $\Ocal(d)\leq \Ocal(d_s n)$ time.

It remains to show that we can implement each iteration in lines (8) and (9)
in $\Ocal(d_s)$ time. To do so, instead of explicitly storing $\bw_t,\bw'_t$,
we only store $\tilde{\bu}$, an auxiliary vector $\bg$, and auxiliary scalars
$\alpha,\beta, \gamma,\delta,\zeta$, such that
\begin{itemize}
    \item At the end of line (8), $\bw'_t$ is stored as
        $\alpha\bg+\beta\tilde{\bu}$
    \item At the end of line (9), $\bw_t$ is stored as
        $\alpha\bg+\beta\tilde{\bu}$
    \item It holds that $\gamma = \norm{\alpha \bg}^2~,~ \delta =
        \inner{\alpha \bg,\tilde{\bu}}~,~ \zeta = \norm{\tilde{\bu}}^2$.
    This ensures that $\gamma+2\delta+\zeta$ expresses
        $\norm{\alpha\bg+\beta\tilde{\bu}}^2$.
\end{itemize}

Before the beginning of the epoch (line (5)), we initialize
$\bg=\tilde{\bw}_{s-1}$, $\alpha=1,\beta=0$ and compute $\gamma =
\norm{\alpha \bg}^2~,~ \delta =\inner{\alpha \bg,\tilde{\bu}}~,~ \zeta =
\norm{\tilde{\bu}}^2$, all in time $\Ocal(d)\leq \Ocal(d_s n)$. This ensures
that $\bw_0 = \alpha \bg+\beta\bu$. Line (8) can be implemented in
$\Ocal(d_s)$ time as follows:
\begin{itemize}
  \item Compute the sparse (on-average) update vector $\Delta \bg
      :=\eta\bx_{i_t}\left(\bx_{i_t}^\top\bw_{t-1}-\bx_{i_t}^\top\tilde{\bw}_{s-1}\right)$
  \item Update $\bg:=\bg+\Delta\bg/\alpha$; $\beta:=\beta+\eta$;
      $\gamma:=\gamma+2\alpha\inner{\bg,\Delta \bg}+\norm{\Delta \bg}^2$;
      $\delta:=\delta+\inner{\Delta \bg,\tilde{\bu}}$. This implements line
      (8), and ensures that $\bw'_t$ is represented as $\alpha \bg+\beta
      \bu$, and its squared norm equals $\gamma+2\delta+\zeta$.
\end{itemize}
To implement line (9), we simply divide $\alpha,\beta$ by
$\sqrt{\gamma+2\delta+\zeta}$ (which equals the norm of $\bw'_t$), and
recompute $\gamma,\delta$ accordingly. After this step, $\bw_t$ is
represented by $\alpha \bg+\beta \tilde{\bu}$ as required.

\end{document}